\documentclass{article}

\usepackage[final]{colm2026_conference}

\usepackage{microtype}
\usepackage{hyperref}
\usepackage{url}
\usepackage{booktabs}

\usepackage{lineno}

\definecolor{darkblue}{rgb}{0, 0, 0.5}
\hypersetup{colorlinks=true, citecolor=darkblue, linkcolor=darkblue, urlcolor=darkblue}

\usepackage{microtype}
\usepackage{graphicx}
\usepackage{subcaption}
\usepackage{booktabs} 
\usepackage{adjustbox} 
\usepackage{algorithm}
\usepackage{algorithmic}
\usepackage{wrapfig}
\usepackage{hyperref}

\usepackage{amsmath}
\usepackage{amssymb}
\usepackage{mathtools}
\usepackage{amsthm}

\usepackage[capitalize,noabbrev]{cleveref}

\theoremstyle{plain}
\newtheorem{theorem}{Theorem}[section]

\theoremstyle{definition}

\theoremstyle{remark}

\newcommand{\weihuanote}[1]{}

\usepackage[textsize=tiny]{todonotes}

\title{GradAlign: Gradient-Aligned Data Selection for LLM Reinforcement Learning}
\author{
\textbf{Ningyuan Yang}$^{1}$\thanks{Equal contribution.}, \ \
\textbf{Weihua Du}$^{2}$\footnotemark[1], \ \
\textbf{Weiwei Sun}$^{2}$, \ \
\textbf{Sean Welleck}$^{2}$, \ \
\textbf{Yiming Yang}$^{2}$ \\
\\
$^{1}$ Institute for Interdisciplinary Information Sciences (IIIS), Tsinghua University \\
$^{2}$ Language Technologies Institute (LTI), Carnegie Mellon University \\
\\
\texttt{yangny23@mails.tsinghua.edu.cn}, \
\texttt{\{weihuad,weiweis,wellecks,yiming\}@cs.cmu.edu}
}
\begin{document}
\newcommand{\METHOD}{{GradAlign}}
\maketitle

\ifcolmsubmission
\linenumbers
\fi





\begin{abstract}
Reinforcement learning (RL) has become a central post-training paradigm for large language models (LLMs), but its performance is highly sensitive to the quality of training problems. This sensitivity stems from the non-stationarity of RL: rollouts are generated by an evolving policy, and learning is shaped by exploration and reward feedback, unlike supervised fine-tuning (SFT) with fixed trajectories. As a result, prior work often relies on manual curation or simple heuristic filters (e.g., accuracy), which can admit incorrect or low-utility problems.
We propose \textbf{\METHOD}, a gradient-aligned data selection method for LLM reinforcement learning that uses a small, trusted validation set to prioritize training problems whose policy gradients align with validation gradients, yielding an adaptive curriculum.
We evaluate \METHOD{} across three challenging data regimes: unreliable reward signals, distribution imbalance, and low-utility training corpus, showing that \METHOD{} consistently outperforms existing baselines, underscoring the importance of directional gradient signals in navigating non-stationary policy optimization and yielding more stable training and improved final performance.
We release our implementation at \url{https://github.com/StigLidu/GradAlign}.
\end{abstract}

\section{Introduction}

Reinforcement learning (RL) has become an important training paradigm for enhancing the capabilities of large language models (LLMs) in complex tasks and environments~\citep{DeepSeekAI2025DeepSeekR1IR,yang2025swe,sun2025training}. However, the current success of LLM RL is critically dependent on high-quality human-labeled training corpora~\citep{hu2025open,he2025deepmath}, which are expensive to obtain. To support large-scale RL training, researchers increasingly turn to massive, automatically collected corpora~\citep{yue2024mammoth2}. These collections inevitably mix high-quality instances with trivial or misleading ones, along with noisy labels. This raises two key challenges: (1) how to identify instances that are suitable for a given downstream task, and (2) how to design mechanisms that prevent low-quality or misleading data from corrupting the RL training signal.

Existing RL data selection strategies for LLM training are relatively simple and primarily rely on rule-based heuristics. For example, \citet{yu2025dapo} retain problems with intermediate accuracy to alleviate gradient vanishing, whereas \citet{sun2025improving} prioritize samples with accuracy near 50\%, analogous to selecting samples near the decision boundary in active learning. Although such partially solved problems can be informative, these methods rely on scalar difficulty signals and do not consider whether the resulting updates move the policy toward improved downstream performance.
Other approaches score training problems using gradient-based criteria, often emphasizing large parameter updates~\citep{li2025limr}, internal consistency within the training distribution~\citep{li2025learnalign}, or gradient diversity~\citep{liang2025can}. However, these strategies do not explicitly verify alignment with downstream optimization objectives and are typically static, failing to adapt the selection criterion as the policy evolves during RL training.

In this work, we propose \textbf{\METHOD}, an automated, gradient-informed algorithm for selecting training problems in LLM reinforcement learning. \METHOD{} leverages a small, trusted validation set and uses the policy gradient as a first-order surrogate to approximate the expected improvement in validation performance induced by policy updates. For each candidate training problem, we compute an alignment score: the cosine similarity between its policy gradient and the aggregated validation gradient. Training problems are then ranked by this alignment score, favoring updates that are most consistent with improving validation performance.

To address RL non-stationarity, \METHOD{} adopts an adaptive curriculum that periodically recomputes validation gradients and re-scores candidate problems under the current policy, using the validation set as a directional probe rather than a direct scoring or retrieval signal.
\METHOD{} is particularly effective in challenging data regimes common in large-scale LLM RL, including unreliable reward signals~\citep{skalse2022defining}, distribution mismatch between training and downstream problems~\citep{kirk2023understanding}, and datasets with low-utility training instances~\citep{dodge2021documenting}. In these settings, scalar accuracy-based heuristics often fail to distinguish useful learning signals from misleading ones, whereas \METHOD{} consistently prioritizes updates that improve downstream performance.

In summary, our contributions are as follows.
\begin{itemize}
    \item We propose \textbf{\METHOD}, an online RL data-selection method that selects training samples whose updates are aligned with downstream performance improvement.
    \item We identify three challenging data regimes for RL data selection and conduct targeted experiments showing that \METHOD{} consistently outperforms prior baselines in these settings.
\end{itemize}
\section{Related Work}

\subsection{Data Selection for LLMs}
Data selection has emerged as a critical lever for enhancing the training efficiency and performance of LLMs~\citep{Koh2017UnderstandingBP}. 
Early approaches largely relied on static heuristics to filter low-quality web text, using metrics such as perplexity thresholds, language identification, deduplication, and rule-based filters~\citep{Penedo2023TheRD, Raffel2019ExploringTL, Abbas2023SemDeDupDL, Li2024DataCompLMIS, Xie2023DataSF, Lin2024Rho1NA}. 
Beyond filtering, optimizing the data-source mix has been standard practice to balance general capabilities with specific domain knowledge~\citep{Xie2023DoReMiOD, Ye2024DataML}. 
More recent strategies employ diversity-based sampling~\citep{Zhang2024HarnessingDF} or proxy models to score data quality~\citep{Wettig2025OrganizeTW, Liu2024RegMixDM}. 
However, these methods predominantly rely on human intuition or general-purpose heuristics rather than a principled evaluation of how specific data points contribute to the model's learning trajectory, often leading to suboptimal alignment with downstream tasks~\citep{Xia2024LESSSI}.

\subsection{Influence Estimation and Gradient Alignment}
To transcend heuristic data selection, researchers have adopted Influence Functions~\citep{Koh2019OnTA, Grosse2023StudyingLL} to theoretically quantify the impact of individual training instances on validation loss. In the context of LLMs, these ideas have inspired a series of gradient-based data selection methods for pre-training and supervised fine-tuning (SFT)~\citep{Engstrom2024DsDmMD, Wang2023FarewellTA}. Specifically, gradient alignment techniques—which select training samples whose gradients exhibit high cosine similarity with those of a trusted validation set—have proven effective for SFT~\citep{Zhang2024HarnessingDF, Xia2024LESSSI, Yu2024MATESMD, Sun2025EnhancingTD}. Recent work has further extended such estimators to preference data, for example, by pruning reward-model training sets using influence functions~\citep{fein2025influence}. However, these methods remain confined to offline, imitation-style paradigms (SFT and reward-model training), and their effectiveness for online RL policy optimization remains largely unexplored.

\subsection{LLM Reinforcement Learning}
Reinforcement Learning (RL) is increasingly central in unlocking advanced LLM capabilities, particularly for complex reasoning~\citep{DeepSeekAI2025DeepSeekR1IR,yu2025dapo} and autonomous agentic tasks~\citep{du2025generalizable}.
While standard Proximal Policy Optimization (PPO)~\citep{Schulman2017ProximalPO} remains a cornerstone, recent advances such as Group Relative Policy Optimization (GRPO)~\citep{shao2024deepseekmath} and the reasoning-focused DeepSeek-R1~\citep{DeepSeekAI2025DeepSeekR1IR} demonstrate that RL can substantially outperform supervised baselines.
However, RL algorithms are notoriously sensitive to data quality and reward noise; for example, training on trivial or unsolvable problems can induce instability or reward hacking~\citep{hu2025open, Gao2022ScalingLF, Lightman2023LetsVS}.
As a result, although considerable effort has been devoted to selecting valuable data using heuristic signals such as success rates, gradients, or parameter distributions~\citep{yu2025dapo, li2025limr,li2025learnalign,liang2025can}, the curriculum of \textit{which} problems to present to LLMs, and \textit{when}, is typically treated as static. This overlooks the potential of dynamic data selection to stabilize and accelerate RL training.
In this work, we propose gradient alignment as a principled approach for dynamically selecting high-quality data during the RL process.
\begin{figure*}[t]
\centering
\vspace{-3mm}
\includegraphics[width=0.95\textwidth]{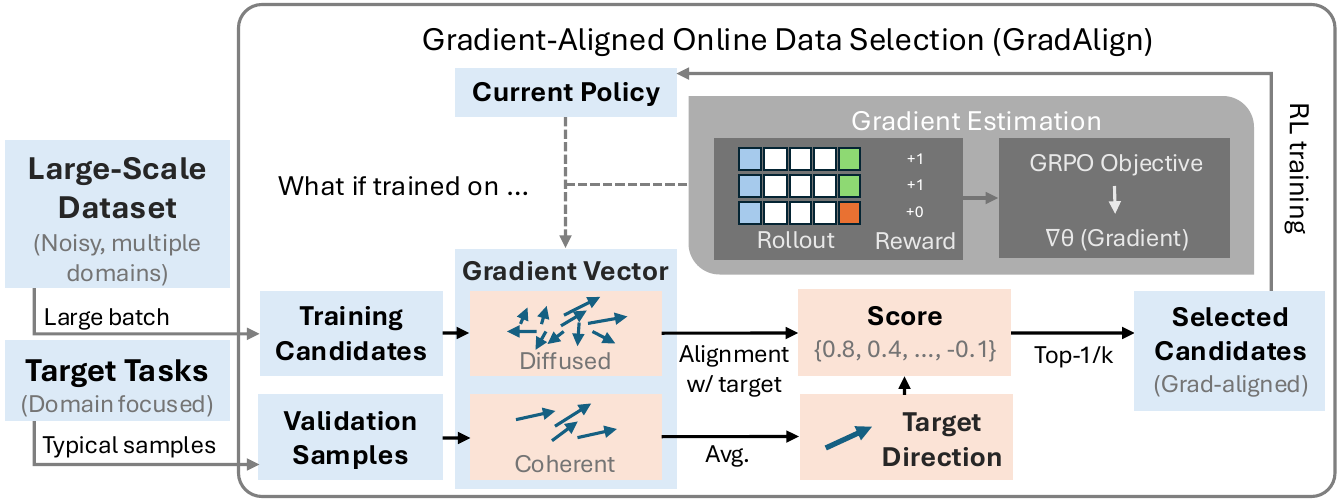}
\vspace{-1mm}
\caption{\textbf{Overview of \METHOD.} \METHOD{} uses a small validation set to estimate a target gradient direction and to score large-scale training candidates by gradient alignment, selecting the top-ranked fraction to form an adaptive RL online learning curriculum.}
\label{fig:teaser}
\vspace{-1mm}
\end{figure*}
\section{Preliminary}

\subsection{Problem Definition}

We study an online reinforcement learning (RL) setting where the learner filters a large, noisy candidate pool to maximize downstream task performance. Formally, at each training round $r$, given the current policy $\pi_\theta$, we receive $M$ candidate problems $\mathcal{P}_r = \{p_{r,1}, \dots, p_{r, M}\}$. A data selector outputs a subset $\mathcal{S}_r \subseteq \mathcal{P}_r$ of size $|\mathcal{S}_r| = M/q$, where $q>1$ is the selection ratio. The policy $\pi_\theta$ is then trained in $\mathcal{S}_r$ using binary rewards computed by a judge model against the reference answers. 

To guide data selection toward informative problems, algorithms can access a small validation set $\mathcal{P}_v$ representing the downstream task (e.g., few-shot examples or public test cases). The final performance is evaluated on a held-out test set $\mathcal{P}_t$.

\subsection{GRPO as the Reinforcement Learning Objective}
\label{sec:grpo}

We use Group Relative Policy Optimization (GRPO)~\citep{shao2024deepseekmath} as a unified training objective to isolate the effect of data selection. For an input $x$ sampled from dataset $\mathcal{D}$, the policy $\pi_\theta$ generates $k$ responses $\{y_j\}_{j=1}^k$ with corresponding rewards $\{r_j\}_{j=1}^k$. GRPO computes a normalized advantage:
$
\hat A_j = \frac{r_j - \bar r}{\sqrt{\frac{1}{k}\sum_{j=1}^{k}(r_j-\bar r)^2}+\epsilon},
$
where $\bar r=\frac{1}{k}\sum_{j=1}^{k} r_j$. The policy is optimized using a clipped objective with KL regularization:
$$
\mathcal{L}_{\text{GRPO}}(\theta) = \mathbb{E}_{x\sim \mathcal{D},\, y\sim\pi_{\theta_{\text{old}}}}\!\left[\min\!\left(\rho_j(\theta)\hat A_j,\, \text{clip}(\rho_j(\theta),1-\epsilon,1+\epsilon)\hat A_j\right)\right] - \beta_{\mathrm{KL}} \mathcal{D}_{\mathrm{KL}}\left[\pi_\theta \,\|\, \pi_{\mathrm{ref}}\right],
$$
where $\rho_j(\theta)= \frac{\pi_\theta(y_j|x)}{\pi_{\theta_{\text{old}}}(y_j|x)}$, and $\beta_{\text{KL}}$ is the coefficient for the KL penalty.
For theoretical analysis, we define a surrogate objective $\mathcal{\tilde L}_{\text{GRPO}}$ that ignores clipping and KL regularization. Assuming on-policy sampling, where the rollout policy is the same as the current policy (i.e., $\pi_{\theta_{\text{old}}} = \pi_\theta$), this yields:
$$
\mathcal{\tilde L}_{\text{GRPO}}(\theta) = \mathbb{E}_{x\sim \mathcal{D},y\sim\pi_{\theta_{\text{old}}}}\left[\hat A_j \log \pi_\theta(y_j|x)\right].
$$
Taking the gradient provides a stable policy-gradient estimator:
$$
\nabla_\theta \mathcal{\tilde L}_{\text{GRPO}}(\theta) = \mathbb{E}_{x\sim \mathcal{D},y\sim\pi_{\theta_{\text{old}}}}[\hat A_j \nabla_\theta \log \pi_\theta(y_j|x)].
$$
While we instantiate our framework using GRPO, it only requires on-policy gradient estimates and is compatible with other RL objectives.

\section{\METHOD}

We propose \textbf{\METHOD}, an online data selection method for RL. While prior work often uses accuracy-based filtering~\citep{yu2025dapo}, accuracy alone cannot determine which problems induce beneficial policy updates. \METHOD{} addresses this limitation by selecting high-quality samples through gradient alignment.

Our key insight is that a training problem is valuable if updating the policy on it improves performance on a trustworthy validation set $\mathcal{P}_v$. We approximate this improvement using first-order policy gradients. \METHOD{} treats the average policy gradient on the validation set as a target direction and ranks candidate problems by their alignment with this direction. As illustrated in Figure~\ref{fig:teaser}, \METHOD{} operates iteratively by estimating validation gradients, selecting aligned training problems, and updating the policy. Rollouts are resampled each round to dynamically adapt to the evolving policy.

\subsection{Validation Improvement as a Proxy}

Unlike supervised fine-tuning, where training data provides detailed solutions, RL only exposes reward signals, making it difficult to directly estimate the utility of individual training problems.
We therefore use \emph{validation improvement} as a proxy for problem utility: a training problem is useful if updating the policy on it improves performance on the validation set $\mathcal{P}_v$, which is more likely to transfer to the held-out test set. A direct approach would be to perform RL updates on each candidate problem and measure the resulting change on $\mathcal{P}_v$, but this is computationally prohibitive. 
\subsection{Validation Improvement via Policy Gradients}

Because accuracy change is a discrete signal, ranking candidates precisely by it is expensive. We instead find that the gradient of the on-policy GRPO loss can serve as a surrogate for accuracy change, as formalized in the following theorem:
\begin{theorem}[Unbiased expected-accuracy gradient estimation]
Under on-policy sampling, binary rewards, unbiased advantage estimation
without normalization, and ignoring KL regularization and clipping,
the policy-gradient estimator is unbiased for the gradient of expected accuracy.
\label{theorem:grpo_gradient_approximation}
\end{theorem}
The proof is provided in Appendix~\ref{sec:proof_grpo_gradient_approximation}.
We therefore approximate the expected change in validation accuracy using the gradient of the GRPO objective. Based on Theorem~\ref{theorem:grpo_gradient_approximation}, under a first-order approximation, the validation improvement induced by a candidate training problem is proportional to the inner product between the policy gradients computed on the validation set and on the candidate.
In practice, the advantage estimator $\hat A$ is normalized for stable training. As the following theorem states, the normalized policy gradient remains directionally aligned, in expectation, with the expected accuracy gradient:
\begin{theorem}[Direction preservation under advantage normalization] Under on-policy sampling, binary rewards, and $k\ge 2$ i.i.d. rollouts per problem, group-wise advantage normalization biases the magnitude of the per-problem GRPO gradient but preserves its expected direction: the expected normalized GRPO gradient is a non-negative scalar multiple of the gradient of the expected accuracy, and the scalar is strictly positive whenever the expected accuracy lies in $(0,1)$. Consequently, whenever $\nabla_\theta p_\theta(x)\neq 0$, the two gradients share the same direction.
\label{theorem:grpo_gradient_direction_correct}
\end{theorem}
The proof is provided in Appendix~\ref{sec:proof_grpo_gradient_direction_correct}. Note that the binary-reward assumption is necessary; Appendix~\ref{app:nonbinary-counterexample} provides a counterexample in the general scalar-reward setting.

Because gradients are typically normalized in practice, their direction is more important than magnitude. Therefore, we use cosine similarity as the alignment score instead of the inner product. This choice is further supported by our ablation studies (see Section~\ref{sec:metric_comparison}).

\subsection{Algorithm Details}

To estimate the validation policy gradient, we first compute the policy gradient for each validation problem. Specifically, for every problem $p_i$ in the validation set $\mathcal{P}_v=\{p^v_1,\ldots,p^v_{|\mathcal{P}_v|}\}$, we sample $k_v$ rollouts $o_{i,1},\ldots,o_{i,k_v}$ with corresponding binary rewards $r_{i,1},\ldots,r_{i,k_v}\in\{0,1\}$. We then compute the normalized advantages $\hat A_{i,j}$ following the GRPO surrogate objective described in Section~\ref{sec:grpo}. The gradient $g_i$ for the problem $p_i$ is given by
\begin{align*}
g^v_i
&=\nabla_\theta \mathcal{\tilde L}_{\text{GRPO}}(p^v_i) 
=\frac{1}{k_v}\sum_{j=1}^{k_v} \hat A_{i,j}\,\nabla_\theta \log \pi_\theta\!\left(o_{i,j}\mid p^v_i\right).
\end{align*}
Averaging all gradients, we obtain the validation policy gradient $G_v=\frac{1}{|\mathcal{P}_v|}\sum_{i=1}g^v_i$. 

In each training round $r$, for each candidate training problem $p_{r,i}\in\mathcal{P}_r$, we similarly estimate the GRPO policy gradient $g_{r,i}$ using sampled rollouts of size $k_r$. Since the training set is substantially larger than the validation set ($|\mathcal{P}_r| \gg |\mathcal{P}_v|$), we use fewer rollouts ($k_r < k_v$) for training candidates, as high-precision gradient estimates are less critical on training candidates. Finally, we rank the candidate training problems according to the cosine similarity between $g_{r,i}$ and $G_v$, selecting the top-$|S_r|$ candidates.

Our design differs from influence functions, which require fully trained models and second-order information~\citep{Koh2017UnderstandingBP}. Second-order effects are unnecessary here, since each candidate is used for only a few updates, for which first-order approximations suffice.

\subsection{On-Policy Resampling}

Because the policy evolves during RL training, both the expected validation accuracy and the validation GRPO gradients change over time. To mitigate the bias introduced by stale rollouts in advantage estimation, we resample all validation problems and recompute the average validation gradient at the beginning of each data-selection round.

This design differs from \citet{zhu2025data}, which estimates gradients using offline rollouts with importance sampling. Although importance sampling can be unbiased, behaviors that emerge as the policy evolves may fall outside the offline rollout distribution. It also exhibits high variance when behavior and target policies diverge, requiring more samples and incurring higher computational cost for stable estimation.
\begin{figure*}[t]
    \centering
    \begin{subfigure}[t]{0.32\textwidth}
        \centering
        \includegraphics[width=\linewidth]{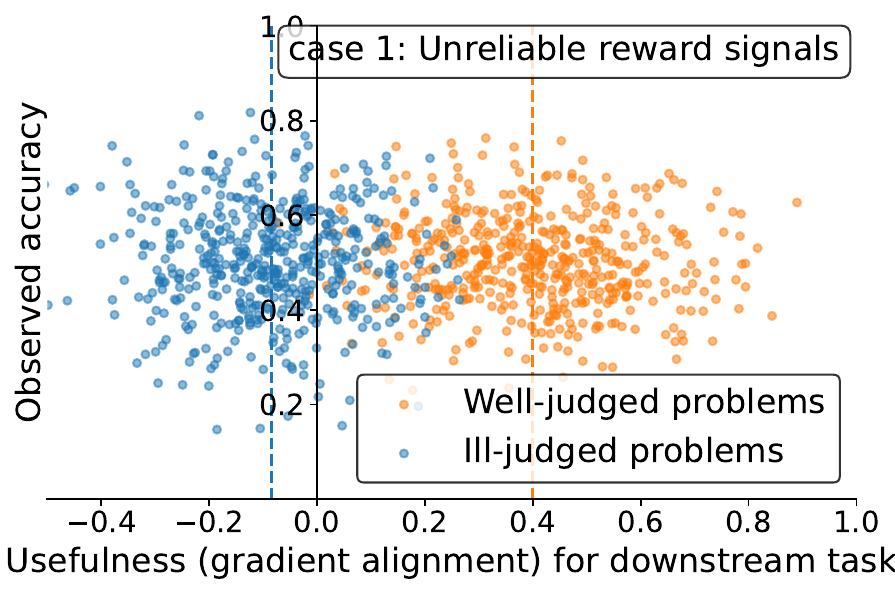}
        \caption{Unreliable reward signals}
        \label{fig:scenario:ill}
    \end{subfigure}
    \hfill
    \begin{subfigure}[t]{0.32\textwidth}
        \centering
        \includegraphics[width=\linewidth]{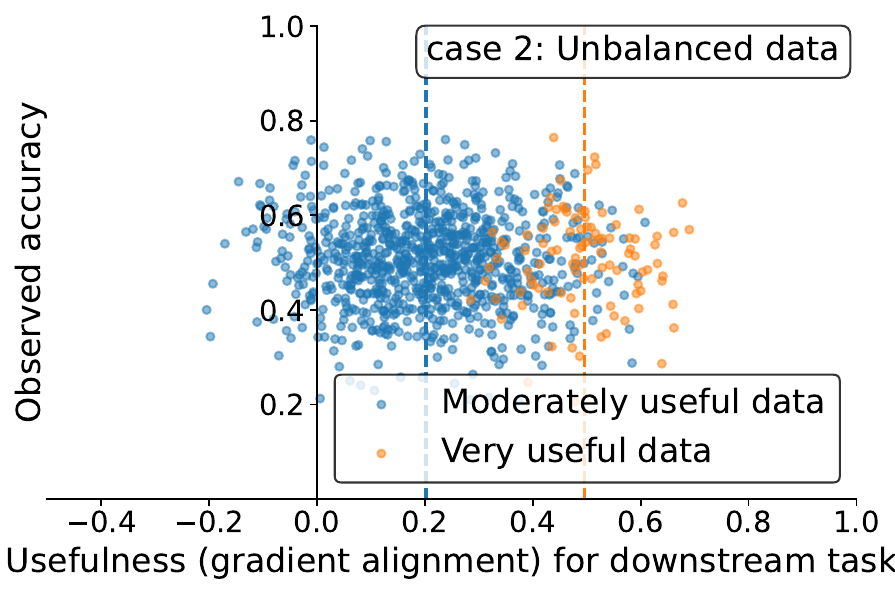}
        \caption{Distribution imbalance}
        \label{fig:scenario:imbalance}
    \end{subfigure}
    \hfill
    \begin{subfigure}[t]{0.33\textwidth}
        \centering
        \includegraphics[width=\linewidth]{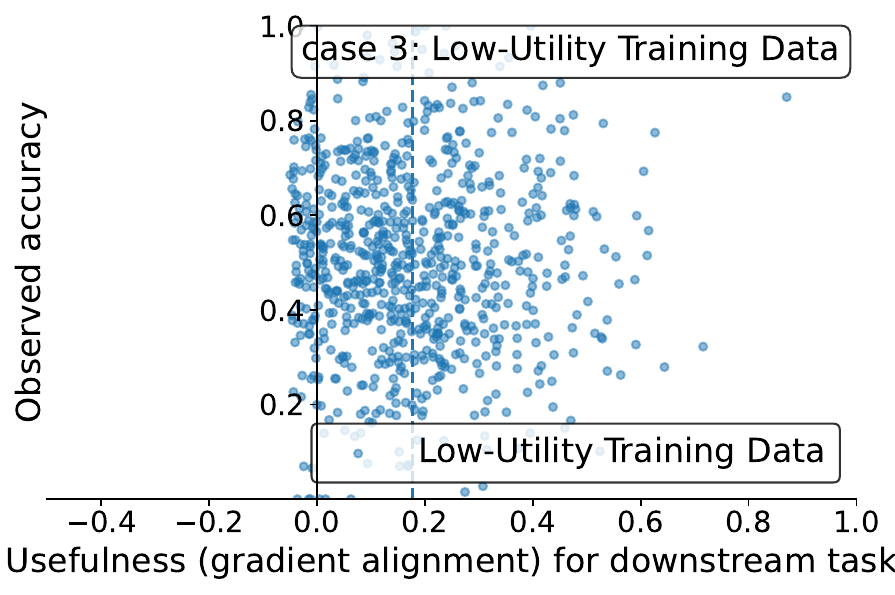}
        \caption{Low-utility training data}
        \label{fig:scenario:auto}
    \end{subfigure}

    \vspace{-2mm}
    \caption{\textbf{Illustration of Three Challenging Data-Selection Scenarios.}
    Each panel shows a failure mode where accuracy-based filtering fails to
    identify training samples that improve downstream performance.}
    \label{fig:scenario}
\end{figure*}

\subsection{Regarding KL Regularization}
Although we do not add KL regularization in our experiments to isolate the effect of data selection, we could modify the surrogate GRPO loss to incorporate a KL penalty for both the training and validation sets. The modified surrogate objective would be formulated as:
\[
\mathcal{\tilde L}_{\text{GRPO-KL}}(\theta)
=
\mathbb{E}_{x,y\sim\pi_\theta}
\left[\hat A_j \log \pi_\theta(y_j|x)\right]
- \beta_{\mathrm{KL}} \mathcal{D}_{\mathrm{KL}} \left[\pi_\theta \,\|\, \pi_{\mathrm{ref}}\right].
\]
In this formulation, we could still compute the gradient alignment by evaluating the gradient of this modified loss, $\nabla_\theta \mathcal{\tilde L}_{\text{GRPO-KL}}(\theta)$, for both the validation set to establish the target direction and the training candidates for ranking.

\paragraph{Overall Algorithm.}
The pseudocode for \METHOD{} is presented in Algorithm~\ref{alg:method}.
\begin{algorithm}[th]
\caption{Grad-Aligned Data Selection~(\textsc{GradAlign})}
\label{alg:method}\small
\begin{algorithmic}[1]
\STATE {\bfseries Input:} initial policy $\pi_\theta$; validation set $\mathcal{P}_v$; number of training rounds $R$; candidate pools $\{\mathcal{P}_r\}_{r=1}^R$ with $|\mathcal{P}_r|=M$; selection ratio $q$ (select $M/q$ candidates per round); number of rollouts $k_v$ and $k_r$; normalization constant $\epsilon$.
\STATE {\bfseries Output:} trained policy $\pi_\theta$
\FOR{$r=1$ to $R$}
    \STATE \textcolor{gray}{\textbf{(Resample \& compute validation gradient)}}
    \STATE Initialize $G_v \leftarrow 0$.
    \FOR{each validation problem $p^v \in \mathcal{P}_v$}
        \STATE Roll out $k_v$ samples $\{o_j\}_{j=1}^{k_v} \sim \pi_\theta(\cdot \mid p^v)$ and obtain rewards $\{r_j\}_{j=1}^{k_v}$.
        \STATE Compute GRPO gradient:\\ $g^v_p \leftarrow \frac{1}{k_v}\sum_{j=1}^{k_v}\hat A_j \nabla_\theta \log \pi_\theta(o_j \mid p^v)$.
        \STATE Accumulate $G_v \leftarrow G_v + g^v_p$.
    \ENDFOR
    \STATE Average $G_v \leftarrow \frac{1}{|\mathcal{P}_v|}G_v$.

    \STATE \textcolor{gray}{\textbf{(Score candidates by gradient alignment)}}
    \FOR{each candidate problem $p_{r,i} \in \mathcal{P}_r$}
        \STATE Roll out $k_r$ samples and compute its GRPO gradient estimate $g_{r,i}$ (same procedure as above).
        \STATE Score $s_{r,i} \leftarrow \cos(g_{r,i}, G_v)$.
    \ENDFOR
    \STATE Select $\mathcal{S}_r \leftarrow$ top-$M/q$ candidates in $\mathcal{P}_r$ by $s_{r,i}$.

    \STATE \textcolor{gray}{\textbf{(RL update on selected curriculum)}} 
    \STATE Update $\theta$ by GRPO training on $\mathcal{S}_r$.
\ENDFOR
\end{algorithmic}
\end{algorithm}

\section{Considerations}

\subsection{Applicable Scenarios}
\label{sec:app-scenario}
\METHOD{} is particularly suitable for settings where the training rewards are noisy or unreliable, as the gradients are more dispersed. We highlight three common scenarios: the \textbf{Distinguishability Test with Noisy Rewards}, an \textbf{Unbalanced Dataset}, and \textbf{Low-Utility Training Data}. Details can be found in Appendix~\ref{app:application} with visualization in Figure~\ref{fig:scenario}.

\subsection{Computational Overhead}
\label{sec:overhead}
\METHOD{} introduces additional computation to estimate gradient alignment, but the relative overhead remains tractable. Let $|\mathcal{P}_r|$ and $|\mathcal{P}_v|$ be the number of training candidates and validation problems, evaluated with $k_r$ and $k_v$ rollouts, respectively. When training a selected subset $|\mathcal{S}_r|$ using $n_t$ rollouts per problem, the relative computational overhead is:
$$ \frac{|\mathcal{P}_r| k_r + |\mathcal{P}_v| k_v}{|\mathcal{S}_r| n_t}. $$
In our standard setting ($|\mathcal{P}_r|=5120$, $|\mathcal{S}_r|=1280$, $k_r=k_v=16$, $n_t=128$, $|\mathcal{P}_v| \le 200$), \METHOD{} adds $\approx 65\%$ compute overhead. By comparison, Align~\citep{li2025learnalign} adds $\approx 50\%$ overhead $\left(\frac{|\mathcal{P}_r| k_r}{|\mathcal{S}_r| n_t}\right)$, and this doubles if gradients are not projected for storage to avoid recomputation. Accuracy-based greedy filtering (AccGreedy, see Section~\ref{sec:baselines}) incurs an overhead of $C \frac{|\mathcal{P}_r| k_r}{|\mathcal{S}_r| n_t}$, where $C$ denotes the relative cost of rollouts compared to full updates.
This overhead is justified for three reasons: (1) the small validation set makes target gradient computation negligible; (2) candidate gradient computation adds only a constant-factor overhead over accuracy-based filtering~\citep{kaplan2020scaling}; and (3) prioritizing highly informative problems improves sample efficiency, significantly reducing ineffective updates and yielding better final performance than unfiltered training. A detailed per-round FLOP breakdown and a memory analysis are provided in Appendix~\ref{app:cost}, confirming that both the compute and memory overhead are invariant to model scale.
\section{Experiments}
\label{sec:experiments}

In this section, we evaluate the effectiveness of {\METHOD} across diverse scenarios. Through these experiments, we aim to verify {\METHOD}'s ability to select high-quality training data for specific domains from a large and noisy training dataset.

\subsection{Experimental Setup}

\paragraph{RL Framework}
All experiments are conducted using GRPO~\citep{shao2024deepseekmath}. We use a sparse binary reward indicating whether the final answer is correct, and all rewards are provided by Qwen2.5-72B-Instruct as the judge. We train base models in the 1.5B--8B range, and the model used in each setting is stated with the corresponding scenario below. Detailed RL training hyperparameters can be found in Appendix~\ref{app:hyper-parameters}. Recalling the three scenarios mentioned in Section~\ref{sec:app-scenario}, we conduct experiments to inspect each of them.

\paragraph{Scenario 1: Distinguishability Test with Noisy Rewards}
To rigorously assess the selector's capacity to discriminate valid learning signals from stochastic noise, we conduct a controlled experiment that systematically perturbs the training set. We train Qwen3-8B-Base on the DAPO training corpus~\citep{yu2025dapo}, introducing controlled corruptions to 50\% of the samples, whereby rewards are decoupled from model outputs and instead sampled from a Bernoulli distribution ($p=0.5$).
This procedure generates a subset of data that statistically mimics the difficulty of ``complex'' problems (50\% pass rate) while being devoid of helpful gradient information.

\paragraph{Scenario 2: Unbalanced Dataset}
To evaluate alignment capabilities under distributional shift, we train Qwen2.5-1.5B-Math on a heterogeneous dataset comprising 4k Countdown Game problems~\citep{gandhi2024stream}, 20k WebInstruct problems, and 10k DAPO problems~\citep{yu2025dapo}. Given that the target domain (Countdown Game) constitutes only $\approx$12\% of the aggregate pool, this setting challenges methods to retrieve domain-specific problems from an imbalanced distribution. We employ a selection ratio of $q=20$ (selecting the top 5\%), with validation and test sets consisting exclusively of Countdown Game problems.

\paragraph{Scenario 3: Low-Utility Training Data}
The training set is the WebInstruct-filtered-unverified dataset~\citep{general-reasoner}, a large-scale collection (808k problems) gathered from internet sources and reflecting typical noisy, uncurated training data. For the two Qwen3-8B-Base settings, we apply a preliminary difficulty filter (pass rates $\in [0.2, 0.8]$) before selection to improve computational efficiency. We consider three target domains.

\textbf{SuperGPQA \& TheoremQA:}
Using Qwen3-8B-Base, we draw a single \emph{combined} validation set from SuperGPQA~\citep{pteam2025supergpqascalingllmevaluation} and TheoremQA~\citep{chen2023theoremqa} to guide selection, and evaluate both held-out test sets; this probes whether a mixed-domain validation signal can improve multiple target domains at once.

\textbf{Mathematical Reasoning Task:}
Using Qwen3-8B-Base, selection is guided by AMC22~\citep{amc22}, with performance evaluated on AMC23 and AIME2425. This is a \emph{partial} validation-test mismatch: the probe is AMC22, but the test set includes the harder AIME2425 from different competitions and years, testing whether {\METHOD} still transfers when the validation set only partially represents the target.

\textbf{MMLU-Pro:}
Using Qwen2.5-1.5B-Math-Instruct and the unfiltered corpus, the target domain is the Math category of MMLU-Pro~\citep{wang2024mmlu}, divided into validation and test sets. This evaluates {\METHOD}'s ability to identify problems beneficial to multiple-choice mathematical reasoning.

\subsection{Baselines}
\label{sec:baselines}
We compare {\METHOD} with three common data selection strategies:
\begin{enumerate}
    \item \textbf{Random Selection:} Uniformly samples training data. Serves as a lower bound.
    \item \textbf{Accuracy Greedy (AccGreedy)}~\citep{sun2025improving}: Prioritizes samples with pass rates closest to 50\%, based on the heuristic that problems near the decision boundary provide the strongest signal.
    \item \textbf{LearnAlign (Align)}~\citep{li2025learnalign}: A method that selects data based on gradient similarity within the training set without using external validation guidance.
\end{enumerate}

\subsection{Experimental Results}

We report results for each scenario below, focusing on downstream accuracy and, where applicable, the composition of selected training data.

\paragraph{Results on Noisy Rewards}

\begin{table*}[t]
  \centering
  \small
  \setlength{\tabcolsep}{5pt}
  \renewcommand{\arraystretch}{1.1}
  \begin{tabular*}{\textwidth}{@{\extracolsep{\fill}}l|c|ccc|cc@{}}
    \toprule
    & \textbf{Val} & \multicolumn{3}{c|}{\textbf{Held-out Test}} & \multicolumn{2}{c}{\textbf{Ratio of Corrupted ($\downarrow$)}} \\
    Method & AMC22 & AMC23 & AIME2425 & \textbf{Test Avg.} & Step 0 & Step 100 \\
    \midrule
    \textbf{\METHOD} (Ours) & \textbf{49.1} & \textbf{68.2} & \textbf{20.5} & \textbf{44.4} & \textbf{17.8\%} & \textbf{29.9\%} \\
    Random & \underline{47.5} & 58.5 & 15.8 & 37.2 & 50.0\%* & 50.0\%* \\
    AccGreedy & 18.1 & 28.5 & 1.9 & 15.2 & 81.8\% & 87.1\% \\
    Align & 45.7 & \underline{61.5} & \underline{16.6} & \underline{39.1} & \underline{33.5\%} & \underline{47.5\%} \\
    \bottomrule
  \end{tabular*}
  \vspace{0.5mm}
  \caption{\textbf{Scenario 1: Controlled Noise Injection (50\% Corrupt).} Downstream performance on AMC/AIME benchmarks and the ratio of corrupted problems in selected training data. Best results are in \textbf{bold}; second-best are \underline{underlined}. All results are reported at step 100. The Test Avg.\ is computed over the held-out test columns only and excludes the validation (Val) column. *Random selection has an expected corruption ratio of 50.0\% by construction.}
  \label{table:noisy_data}
  \vspace{-4mm}
\end{table*}

Table~\ref{table:noisy_data} presents the results of the noise-injection experiment. {\METHOD} significantly outperforms all baselines on most datasets. 
We also report the ratio of corrupted problems selected at steps 0 and 100 in Table~\ref{table:noisy_data}. {\METHOD} consistently selects the lowest proportion of corrupted data. AccGreedy selects the most corrupted data (over 80\%) because the corrupted samples have pass rates near 50\%. Figure~\ref{fig:noisy_curve} shows the training accuracy curves, where {\METHOD} outperforms all baselines.

\paragraph{Results on Unbalanced Dataset}

Table~\ref{table:countdown} shows results for the unbalanced dataset setting. {\METHOD} achieves the best performance on both validation and test sets, followed by Align. Table~\ref{table:countdown} shows the ratio of the selected Countdown Game problems. AccGreedy rarely selects these problems because of the low initial accuracy on Countdown. This result demonstrates the importance of gradient-informed data selection for domain-specific filtering.

\begin{table*}[t]
  \centering
  \small
  \setlength{\tabcolsep}{4pt}
  \renewcommand{\arraystretch}{1.1}
  \begin{tabular*}{\textwidth}{@{\extracolsep{\fill}}l|cc|cc@{}}
    \toprule
    & \multicolumn{2}{c|}{\textbf{Performance}} & \multicolumn{2}{c}{\textbf{Ratio of Countdown ($\uparrow$)}} \\
    Method & Countdown(Val) & Countdown(Test) & Step 0 & Step 50 \\
    \midrule
    \textbf{\METHOD} (Ours) & \textbf{33.8} & \textbf{34.0} & \textbf{92.9\%} & \textbf{62.1\%} \\
    Random & 14.1 & 10.4 & 11.8\%* & 11.8\%* \\
    AccGreedy & 21.7 & 15.2 & 7.8\% & 1.5\% \\
    Align & \underline{31.1} & \underline{28.2} & \underline{47.2\%} & \underline{60.6\%} \\
    \bottomrule
  \end{tabular*}
  \vspace{0.5mm}
  \caption{\textbf{Scenario 2: Unbalanced Dataset.} Downstream performance on Countdown Game and the ratio of Countdown problems in selected training data. Best results are in \textbf{bold}; second-best are \underline{underlined}. All results are reported at step 50. *Random selection has an expected Countdown ratio of 11.8\% by construction.}
  \label{table:countdown}
  \vspace{-1mm}
\end{table*}

\begin{table*}[t]
  \centering
  \small
  \setlength{\tabcolsep}{3pt}
  \renewcommand{\arraystretch}{1.1}
  \begin{tabular*}{\textwidth}{@{\extracolsep{\fill}}l|cc|cc|ccc|cc|c@{}}
    \toprule
    & \multicolumn{2}{c|}{\textbf{SuperGPQA}} & \multicolumn{2}{c|}{\textbf{TheoremQA}} & \multicolumn{3}{c|}{\textbf{AIME/AMC}} & \multicolumn{2}{c|}{\textbf{MMLU-Pro}} & \textbf{Test}\\
    Method & Val & Test & Val & Test & Val & AMC23 & AIME2425 & Val & Test & \textbf{Avg.} \\
    \midrule
    \textbf{\METHOD} & \textbf{33.5} & \textbf{35.7} & \textbf{55.7} & \textbf{59.3} & \textbf{48.8} & \textbf{64.2} & \textbf{18.3} & \textbf{59.8} & \textbf{57.9} & \textbf{47.1} \\
    Random & 29.6 & \underline{35.5} & \underline{53.5} & 55.2 & \underline{47.6} & 60.0 & 17.2 & \underline{57.8} & \underline{57.1} & \underline{45.0} \\
    AccGreedy & \underline{31.7} & 31.6 & 53.1 & 54.9 & 45.0 & \underline{62.3} & \underline{17.8} & 57.5 & 55.1 & 44.3 \\
    Align & 29.8 & 32.7 & 52.8 & \underline{57.9} & 46.4 & 61.5 & 16.1 & 57.1 & 56.6 &  \underline{45.0}  \\
  \bottomrule
  \end{tabular*}
  \vspace{-0.5em}
  \caption{\textbf{Scenario 3: Low-Utility Training Data.} Downstream performance on SuperGPQA, TheoremQA, AMC/AIME, and MMLU-Pro. The training set is WebInstruct.
  Best results are in \textbf{bold}; second-best are \underline{underlined}. All results are reported at step 100. The Test Avg.\ covers only the held-out test columns, excluding all Val columns.}
  \label{table:gpqa_math}
\end{table*}

\paragraph{Performance on Low-Utility Data}

WebInstruct is a large-scale, automatically constructed corpus in which many problems are on-distribution but overly simple, providing valid rewards yet little learning signal for the target task. On this corpus, {\METHOD} attains the highest average performance across all settings, while some baselines remain competitive on individual splits (Table~\ref{table:gpqa_math}). It achieves the best results across all datasets, e.g., 59.3\% on TheoremQA-test, compared to 54.9\% for AccGreedy and 57.9\% for Align.

Although {\METHOD} requires more computation per step, it converges faster and reaches higher final performance than random selection. As shown in Figure~\ref{fig:math_compare}, in Scenario 3, random selection does not plateau but degrades after step 120, as continued training on noisy data erodes the policy, so additional training would not close the gap.

\begin{figure*}[th]
\centering
\vspace{-1mm}
\includegraphics[width=0.9\linewidth]{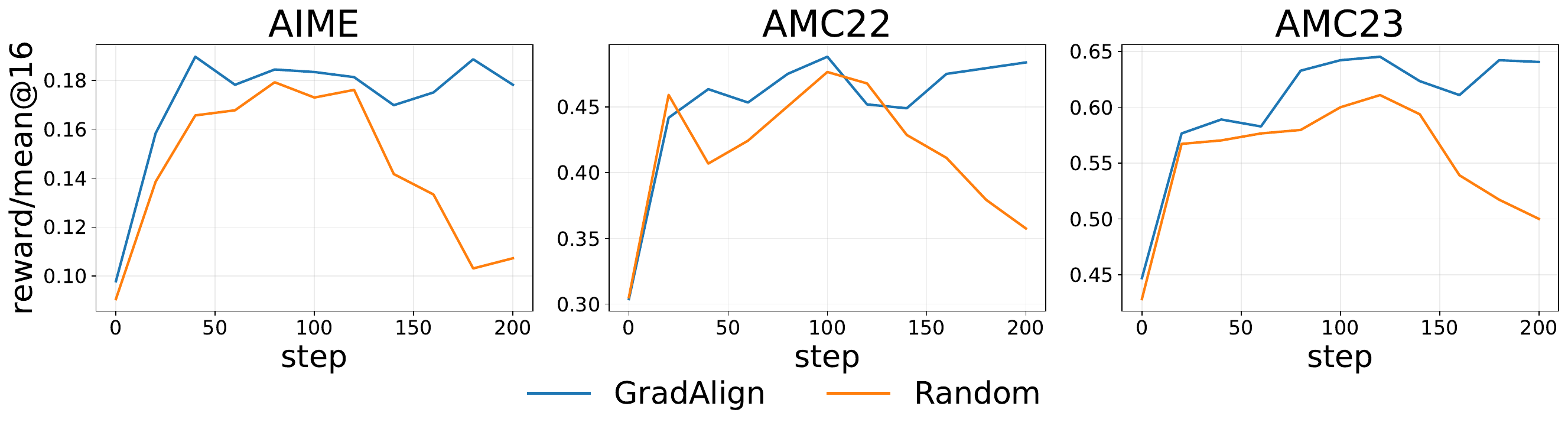}
\vspace{-2mm}
\caption{\textbf{Training Accuracy Curve on AIME2425, AMC22 and AMC23 (Scenario 3).} {\METHOD} (ours) achieves the strongest performance.}
\label{fig:math_compare}
\end{figure*}

\section{Ablation Study}
\label{sec:metric_comparison}

In this section, we conduct ablation studies on three design choices: the alignment score metric (cosine similarity versus inner product), the sample size required for stable gradient estimation, and training directly on the validation set instead of the training set.

\paragraph{Inner product vs. cosine similarity.} 
We compare cosine similarity and the inner product as alignment-score measures in the noisy-reward setting. As shown in Figure~\ref{fig:ablation} (left), cosine similarity achieves higher performance. Further analysis (see Appendix~\ref{app:dis_by_score}) indicates that cosine similarity better separates clean samples from corrupted samples.

\begin{wrapfigure}{r}{0.55\textwidth}
    \centering
    \vspace{-1em}
    \includegraphics[width=\linewidth]{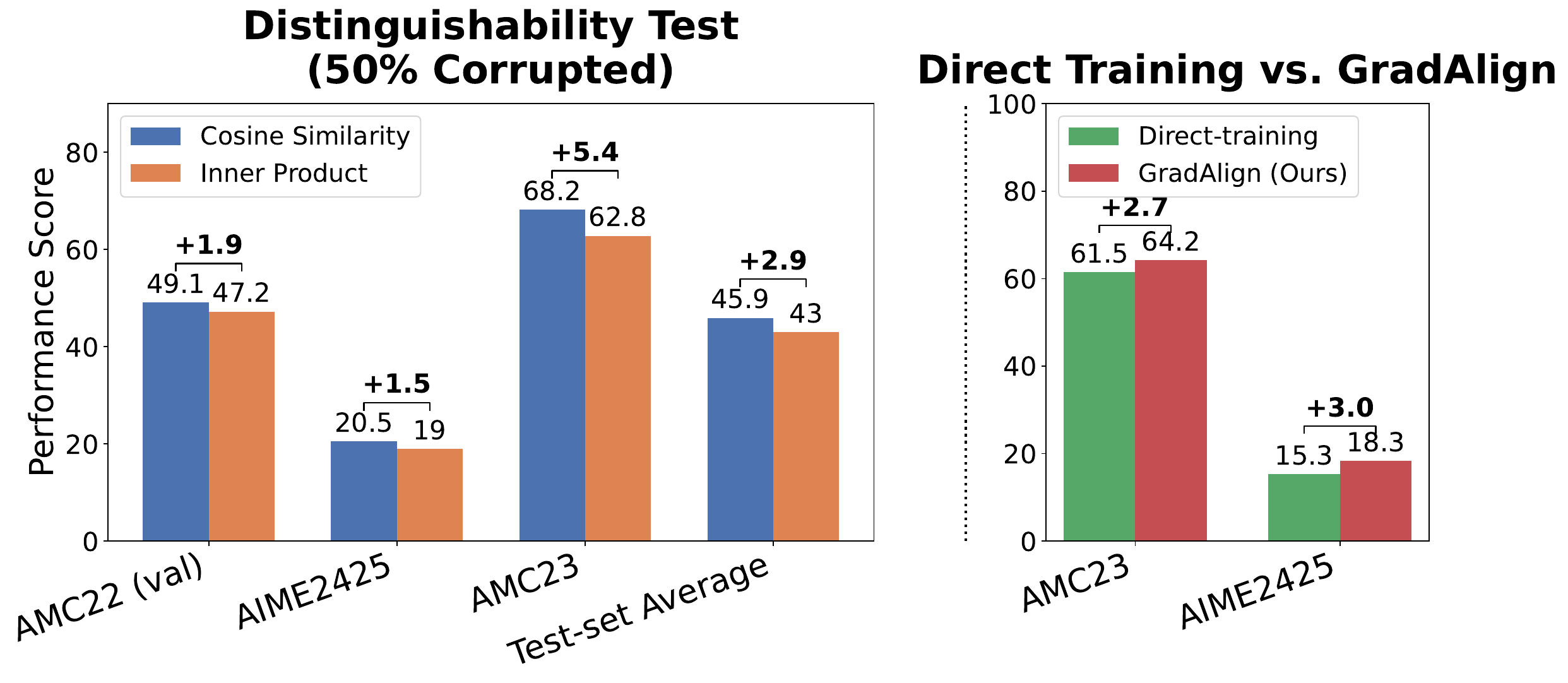}
    \vspace{-1em}
    \caption{\footnotesize \textbf{Performance Analysis of Gradient Alignment and Direct Training.} (Left) Under 50\% reward noise, cosine similarity outperforms the inner product. (Right) Direct training on the validation set overfits, whereas \METHOD{} generalizes. All results are reported at step 100.}
    \label{fig:ablation}
    \vspace{-1em}
\end{wrapfigure}

\paragraph{Directly training on validation set.} 
We study the performance of training directly on the validation set (AMC22) instead of selecting data from the training set. Figure~\ref{fig:ablation} (right) shows that while direct training achieves higher performance on AMC22 itself, it fails to generalize to test sets, unlike {\METHOD}. Intuitively, training directly on the small validation set can lead to memorizing its specific trajectories, whereas {\METHOD} preserves diversity.

\paragraph{Validation set size}
To test the sensitivity of our method to the size of the validation set, we conduct an ablation study in the \textit{Unbalanced dataset} setting. As shown in Table~\ref{table:ablation}, {\METHOD} achieves a stable improvement over baselines even with small validation set sizes. We also measure gradient noise in GRPO, and results are given in Appendix~\ref{app:measure}.
\begin{table*}[h]
  \centering
  \small
  \renewcommand{\arraystretch}{1.1}
  \begin{tabular}{lccc|ccc}
    \toprule
    & \multicolumn{3}{c|}{\textbf{\METHOD} (val.\ set size)} & \multicolumn{3}{c}{\textbf{Baselines}} \\
    \cmidrule(lr){2-4} \cmidrule(lr){5-7}
    & \textbf{10} & \textbf{30} & \textbf{100} & \textbf{Random} & \textbf{AccGreedy} & \textbf{Align} \\
    \midrule
    \textbf{Test Set Accuracy} & 32.7 & 32.1 & \textbf{34.0} & 10.4 & 15.2 & 28.2 \\
    \bottomrule
  \end{tabular}
  \caption{Test set performance at step 50 on the Countdown task, varying the \METHOD{} validation set size (10/30/100 problems). \METHOD{} stays well above all baselines.}
  \label{table:ablation}
\end{table*}

\section{Conclusion}
We presented {\METHOD}, a gradient-aligned data selection method for LLM reinforcement learning that uses a small validation set to prioritize training problems whose policy-gradient updates most improve downstream performance. Across noisy rewards, distribution imbalance, and low-utility web data, {\METHOD} yields more stable training and stronger final results than heuristic and prior baselines. These results suggest that policy-gradient direction is a reliable signal for RL data attribution and curriculum learning, offering a broadly applicable way to improve training efficiency in LLM reinforcement learning.

\paragraph{Limitations}
{\METHOD} relies on a validation set representative of the downstream task: it selects training samples aligned with the validation distribution, so a mismatched probe (e.g., validating on MMLU but testing on AIME) would steer selection toward the wrong problems. It thus targets settings with a relevant validation signal rather than task-agnostic selection. It also adds compute overhead ($\approx 65\%$ over unfiltered training; \cref{sec:overhead}), though this cost is bounded and scale-invariant.

\section*{Acknowledgments}
This work was supported in part by the National Science Foundation under Grant No. DMS-2502281.


\bibliography{main}
\bibliographystyle{icml2026}

\appendix
\onecolumn

\section{Detailed Hyperparameters}
\label{app:hyper-parameters}

We detail the hyperparameters used for GRPO training and the \METHOD{} selection process in Table~\ref{tab:hyperparams}. For all experiments, we utilize the AdamW optimizer with a learning rate of $1 \times 10^{-6}$.

\begin{table*}[th]
    \centering
    \renewcommand{\arraystretch}{1.2}
    \resizebox{\textwidth}{!}{%
    \begin{tabular}{lccccc}
    \toprule
    \textbf{Hyperparameter} & \textbf{MMLU-Pro} & \textbf{\shortstack{SuperGPQA \& \\ TheoremQA}} & \textbf{\shortstack{Math \\ (AIME/AMC)}} & \textbf{Countdown} & \textbf{\shortstack{Noisy Rewards \\ (Distinguishability)}} \\
    \midrule
    \multicolumn{6}{l}{\textit{Model Configuration}} \\
    \midrule
    Base Model & Qwen2.5-1.5B-Math-Instruct & Qwen3-8B-Base & Qwen3-8B-Base & Qwen2.5-1.5B-Math & Qwen3-8B-Base \\
    \midrule
    \multicolumn{6}{l}{\textit{GRPO Training}} \\
    \midrule
    Learning Rate & $1 \times 10^{-6}$ & $1 \times 10^{-6}$ & $1 \times 10^{-6}$ & $1 \times 10^{-6}$ & $1 \times 10^{-6}$ \\
    Optimizer & AdamW & AdamW & AdamW & AdamW & AdamW \\
    Training Problems per Step ($n_{\text{train}}$) & 32 & 128 & 128 & 128 & 128 \\
    Rollouts per Training Problem & 128 & 128 & 128 & 128 & 128 \\
    \midrule
    \multicolumn{6}{l}{\textit{Data Selection}} \\
    \midrule
    Selection Interval ($U$ steps) & 10 & 10 & 10 & 10 & 10 \\
    Rollouts for Estimation ($k_v$) & 16 & 16 & 16 & 16 & 64 \\
    Selection Ratio ($q$) & 4 & 4 & 4 & 20 & 4 \\
    Selected Percentage & 25\% & 25\% & 25\% & 5\% & 25\% \\
    \bottomrule
    \end{tabular}%
    }
    \caption{\textbf{Detailed Hyperparameters across Experimental Settings.} Note that $n_{\text{train}}$ denotes the number of distinct problems per GRPO update step. The selection ratio $q$ determines the fraction of the candidate pool retained ($1/q$).}
    \label{tab:hyperparams}
\end{table*}

\begin{figure*}[th]
\centering
\includegraphics[width=1.0\linewidth]{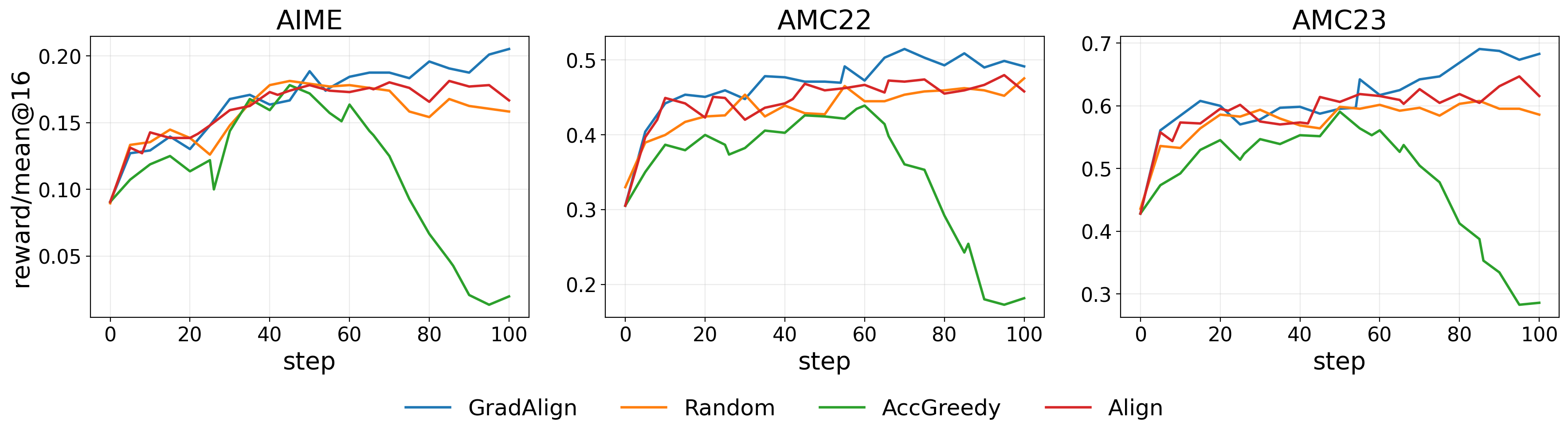}
\caption{\textbf{Training Accuracy Curve on AIME2425, AMC22 and AMC23 (Scenario 1).} {\METHOD} (ours) achieves the strongest performance.}
\label{fig:noisy_curve}
\end{figure*}

\weihuanote{
\begin{table}[H]
  \centering
  \resizebox{0.3\textwidth}{!}{
  \begin{tabular}{lcc}
    \toprule
    Dataset & {\METHOD} & Inner Product  \\
    \midrule
    \multicolumn{3}{c}{\textit{Ratio of Corrupted problems chosen}} \\
    Step 0      & \textbf{17.8\%} & 27.6\%   \\
    Step 50     & \textbf{37.1\%} & 60.3\%   \\
    Step 100     & \textbf{29.9\%} & 44.3\%  \\
    \bottomrule
  \end{tabular}
  }
  \caption{\textbf{Ratio of Corrupted Problems in Selected Data.} Cosine similarity provides better guidance than inner product.
  }\label{table:noisy_data_ratio_ablation}
\end{table}}

\begin{figure}[th]
\centering
\includegraphics[width=0.5\linewidth]{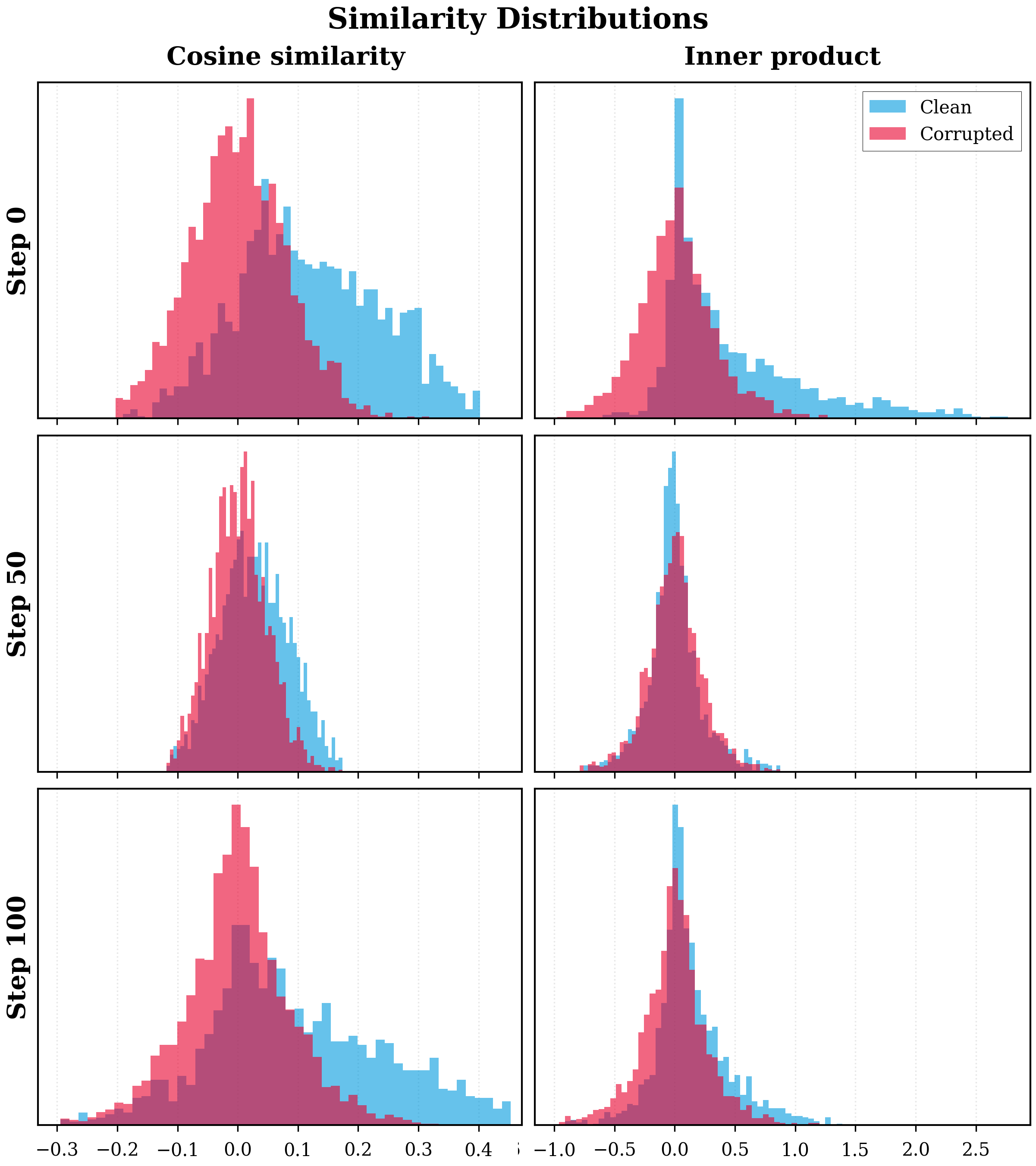}
\caption{\textbf{Distribution of Cosine Similarity and Inner Product Similarity.} Cosine similarity is more indicative than the inner product regarding detecting corrupted instances.}
\label{fig:similarity_dist}
\end{figure}

\section{Candidate Distribution}
\label{app:dis_by_score}
We visualize the score distribution for Scenario 1 (Noisy Rewards), in which half of the training candidates are corrupted by random reward signals sampled from a Bernoulli distribution with $p = 0.5$. An effective data selection method should be able to distinguish clean candidates from corrupted ones. As shown in Figure~\ref{fig:similarity_dist}, we compare two similarity metrics (cosine similarity and inner product) for computing alignment scores. Cosine similarity provides clearer separation between clean and corrupted candidates across different training stages, demonstrating its superior discriminative capability.

\section{Proofs of Validation Improvement}
\label{sec:deferred_proofs}

\subsection{Proof of Theorem~\ref{theorem:grpo_gradient_approximation}}
\label{sec:proof_grpo_gradient_approximation}

\begin{proof}
Let $J(\theta)=\mathbb{E}_{x\sim\mathcal D,\,y\sim\pi_\theta(\cdot|x)}[R(x,y)]$ be the expected reward, where $R(x,y)=\mathbf{1}\!\left[y = y^*(x)\right]$ is the judge outcome indicating whether $y$ matches the reference answer $y^*(x)$. Hence, $J(\theta)$ equals the expected accuracy.
Under on-policy sampling, the policy gradient theorem~\citep{sutton1999policy} gives
\[
\nabla_\theta J(\theta)=\mathbb{E}_{x,y\sim\pi_\theta}\!\left[Q^\pi(x,y)\nabla_\theta\log\pi_\theta(y|x)\right].
\]
Here $Q^\pi(x,y)$ denotes the action-value function, i.e., the expected
return obtained by taking action $y$ under policy $\pi$ given input $x$.
Let $A^\pi(x,y)=Q^\pi(x,y)-b(x)$ denote the advantage function,
where $b(x)$ is any baseline independent of $y$. Then
$\mathbb{E}_{y\sim\pi_\theta}
[b(x)\nabla_\theta\log\pi_\theta(y|x)] = 0.
$
Therefore
\[
\nabla_\theta J(\theta)
=\mathbb{E}_{x,y\sim\pi_\theta}
\!\left[A^\pi(x,y)\nabla_\theta\log\pi_\theta(y|x)\right].
\]
If the advantage estimator $\hat A$ is unbiased, i.e.,
$\mathbb{E}[\hat A \mid x,y] = A^\pi(x,y)$, then
\[
\mathbb{E}_{x,y\sim\pi_\theta}
\!\left[\hat A \nabla_\theta \log \pi_\theta(y\mid x)\right]
=
\nabla_\theta J(\theta),
\]
showing that the resulting update is an unbiased estimator of $\nabla_\theta J(\theta)$, which equals the gradient of expected accuracy.

We note that GRPO does not use such an unbiased advantage: it centers rewards by the group mean $\bar r=\frac1k\sum_{\ell=1}^k r_\ell$, which includes the sampled reward and is therefore not a baseline independent of the action. Writing $\bar r_{-j}=\frac{1}{k-1}\sum_{\ell\neq j} r_\ell$ for the leave-one-out mean, which \emph{is} a valid action-independent baseline, one has $r_j-\bar r=\big(1-\frac1k\big)\big(r_j-\bar r_{-j}\big)$. Hence group-mean centering (without standard-deviation normalization) equals the unbiased leave-one-out estimator scaled by $1-\frac1k$: it is biased in magnitude but preserves the direction of the expected-accuracy gradient.
\end{proof}

\subsection{Proof of Theorem~\ref{theorem:grpo_gradient_direction_correct}}
\label{sec:proof_grpo_gradient_direction_correct}

\begin{proof}
Fix a problem $x$, and define its expected accuracy under the current policy as
\[p_\theta(x) = \mathbb{E}_{y\sim\pi_\theta(\cdot\mid x)}\left[R(x,y)\right],\] where \(R(x,y)\in\{0,1\}\), and thus $p_\theta(x) \in [0, 1]$.

Let \(y_1,\ldots,y_k\overset{\mathrm{i.i.d.}}{\sim}\pi_\theta(\cdot\mid x),r_j=R(x,y_j), p = p_\theta(x)\), and define the score vector \(s_j=\nabla_\theta\log\pi_\theta(y_j\mid x).\)

As is standard in policy-gradient methods, the sampled rewards and normalization statistics are treated as stop-gradient quantities when differentiating the GRPO surrogate objective. We define the group mean \(\bar r=\frac{1}{k}\sum_{j=1}^k r_j\) and, for any normalization constant $\epsilon\ge 0$, the normalized advantage
\(\widehat A_j=\frac{r_j-\bar r}{\sqrt{\frac{1}{k}\sum_{\ell=1}^k(r_\ell-\bar r)^2+\epsilon}},\) with the convention that a zero-variance group (in which all rewards are equal) contributes zero advantage, $\widehat A_j=0$; this covers the case $\epsilon=0$, for which the denominator would otherwise vanish. The corresponding normalized GRPO gradient estimator per-problem is
\[\widehat g_{\mathrm{norm}}(x)=\frac{1}{k}\sum_{j=1}^k\widehat A_j s_j.\]
We first establish two standard score-function identities:
\begin{align*}
\mathbb{E}[s_j]&=\mathbb{E}_{y\sim\pi_\theta(\cdot\mid x)}\left[\nabla_\theta\log\pi_\theta(y\mid x)\right]=0,\\
\mathbb{E}[r_j s_j]
&=
\mathbb{E}_{y\sim\pi_\theta(\cdot\mid x)}
\left[
R(x,y)\nabla_\theta\log\pi_\theta(y\mid x)
\right]
=
\nabla_\theta
\mathbb{E}_{y\sim\pi_\theta(\cdot\mid x)}
\left[R(x,y)\right]
=
\nabla_\theta p.
\end{align*}
If $p\in\{0,1\}$, all sampled rewards are equal, so $\widehat g_{\mathrm{norm}}(x)=0$ and the claim holds with $c_\theta(x)=0$. We therefore assume $p\in(0,1)$ below.
Because the reward is binary, we have \(\mathbb{E}[r_j s_j] = p\mathbb{E}[s_j\mid r_j=1]\). Therefore, $\mathbb{E}[s_j\mid r_j=1] = \frac{\nabla_\theta p}{p}.$ Similarly, for $\mathbb{E}[s_j\mid r_j=0]$, we have
\begin{align*}
(1-p)\mathbb{E}[s_j\mid r_j=0]=\mathbb{E}[(1-r_j)s_j] = \mathbb{E}[s_j]-\mathbb{E}[r_j s_j] = -\nabla_\theta p,
\end{align*}
so \(\mathbb{E}[s_j\mid r_j=0] = -\frac{\nabla_\theta p}{1-p}.\) 

Let $n=\sum_{i=1}^{k}r_i = \bar{r}k$ denote the number of correct rollouts in this realized group; the group variance is
\[
\frac{1}{k}\sum_{\ell=1}^k(r_\ell-\bar r)^2=
\frac{n}{k}\left(1-\frac{n}{k}\right)^2+\frac{k-n}{k}\left(\frac{n}{k}\right)^2
=
\frac{n}{k}
\left(1-\frac{n}{k}\right).
\]
Define $d_n = \sqrt{\dfrac{n}{k}\left(1-\dfrac{n}{k}\right)+\epsilon}$. For $n\in\{1,\ldots,k-1\}$ this is strictly positive; the boundary cases $n\in\{0,k\}$ are zero-variance groups that contribute zero by the convention above (their numerator $n(k-n)$ also vanishes), so the value of $d_n$ there is immaterial. Condition on an arbitrary reward vector \((r_1,\ldots,r_k)\) containing exactly \(n\) correct rollouts. Since the rollouts are independent, the conditional expectation of each score vector depends only on its corresponding reward. Therefore,
\begin{align*}
\mathbb{E}
\left[
\sum_{j=1}^k
(r_j-\bar r)s_j
\;\middle|\;
r_1,\ldots,r_k
\right]
&=
n\left(1-\frac{n}{k}\right)
\mathbb{E}[s_j\mid r_j=1]
+
(k-n)\left(-\frac{n}{k}\right)
\mathbb{E}[s_j\mid r_j=0]\\
&=n\left(1-\frac{n}{k}\right)
\frac{\nabla_\theta p}{p}
+
(k-n)\left(-\frac{n}{k}\right)
\left(
-\frac{\nabla_\theta p}{1-p}
\right)\\
&=\frac{n(k-n)}{k}
\left(
\frac{1}{p}
+
\frac{1}{1-p}
\right)
\nabla_\theta p\\
&=
\frac{n(k-n)}
{kp(1-p)}
\nabla_\theta p.
\end{align*}
This expression depends on the reward vector only through $n$. Let $N$ be the number of correct rollouts. Conditional on $N=n$, the expected normalized gradient is
\begin{align*}
\mathbb{E}
\left[
\widehat g_{\mathrm{norm}}(x)
\mid N=n,x
\right] = \frac{1}{kd_n}\mathbb{E}
\left[
\sum_{j=1}^k
(r_j-\bar r)s_j
\;\middle|\;
N=n, x
\right] = \frac{n(k-n)}
{k^2p(1-p)d_n}
\nabla_\theta p.
\end{align*}
Taking expectation over \(N\sim\operatorname{Binomial}(k,p)\) gives
\begin{align*}
\mathbb{E}
\left[
\widehat g_{\mathrm{norm}}(x)
\mid x
\right]
=
\sum_{n=0}^k
\Pr(N=n)
\frac{n(k-n)}
{k^2p(1-p)d_n}
\nabla_\theta p
=
c_\theta(x)
\nabla_\theta p_\theta(x),
\end{align*}
where 
\begin{align*}
c_\theta(x)=
\sum_{n=0}^{k}
\binom{k}{n}
p^n(1-p)^{k-n}
\frac{n(k-n)}
{k^2p(1-p)d_n}.
\end{align*}
Here the boundary terms $n\in\{0,k\}$ are zero (their numerator $n(k-n)$ vanishes, and for $\epsilon=0$ they are zero by the zero-variance convention above). All the remaining terms are non-negative, and when $p \not= 0,1$ and $ n \not= 0,k$ they are strictly positive. Therefore, $c_\theta(x)$ is positive when $p\in (0,1)$ and $k \geq 2$.

It follows that \[\mathbb{E}
\left[
\widehat g_{\mathrm{norm}}(x)
\mid x
\right]
=
c_\theta(x)
\nabla_\theta p_\theta(x),
\qquad
c_\theta(x)\ge 0,\] where \(c_\theta(x)>0\) whenever \(p_\theta(x)\in(0,1)\) and \(k\ge 2\).
\end{proof}


\subsection{Counterexample for Non-Binary Rewards}
\label{app:nonbinary-counterexample}

The binary-reward assumption in Theorem~\ref{theorem:grpo_gradient_direction_correct} is
essential. We construct a counterexample showing that, for general scalar
rewards, group-wise standard-deviation normalization does not necessarily
preserve the expected policy-gradient direction.

Consider a softmax policy over three actions parameterized by logits
$\theta\in\mathbb{R}^3$,
\[
\pi_\theta(i)
=
\frac{\exp(\theta_i)}
{\sum_{j=1}^3\exp(\theta_j)},
\]
evaluated at the uniform policy
\[
\pi_\theta(1)
=
\pi_\theta(2)
=
\pi_\theta(3)
=
\frac13.
\]
Let
\[
s_i
=
\nabla_\theta\log\pi_\theta(i)
\]
denote the score function associated with action $i$. Under the softmax
parameterization,
\[
s_i=e_i-\pi_\theta,
\]
where $e_i$ is the $i$-th standard basis vector. Assign deterministic
rewards
\[
R(1)=0,\qquad
R(2)=1,\qquad
R(3)=3.
\]

The expected reward is
\[
J(\theta)
=
\mathbb{E}_{i\sim\pi_\theta}[R(i)],
\]
whose policy gradient is
\[
\nabla_\theta J(\theta)
=
\sum_{i=1}^3
\pi_\theta(i)R(i)s_i
=
\frac13 s_2+s_3
=
\frac19
\begin{pmatrix}
-4\\
-1\\
5
\end{pmatrix}.
\]

Now consider group-wise advantage normalization with $k=2$ independent
rollouts in the limit $\epsilon\to0^+$. Zero-variance groups have zero
gradient because their rewards are equal. For two actions with distinct rewards,
\[
\widehat g_{\mathrm{norm}}(i,j)
\longrightarrow
\frac12
\operatorname{sign}(R(i)-R(j))
(s_i-s_j).
\]

The key here is that standard-deviation normalization
removes the magnitude of each pairwise reward difference and retains only
its sign. At the uniform policy, each ordered action pair is sampled with
probability $1/9$. Combining the two orderings of each distinct pair gives
\[
\begin{aligned}
\mathbb{E}
\left[
\widehat g_{\mathrm{norm}}
\right]
&=
\frac19
\left[
-(s_1-s_2)
-(s_2-s_3)
-(s_1-s_3)
\right] \\
&=
\frac19
\begin{pmatrix}
-2\\
0\\
2
\end{pmatrix},
\end{aligned}
\]
where we use
\[
s_i-s_j=e_i-e_j.
\]

Since
\[
\frac19
\begin{pmatrix}
-2\\
0\\
2
\end{pmatrix}
\not\parallel
\frac19
\begin{pmatrix}
-4\\
-1\\
5
\end{pmatrix},
\]
the expected normalized policy gradient is not a scalar multiple of
$\nabla_\theta J(\theta)$. Therefore, group-wise standard-deviation
normalization does not, in general, preserve the expected policy-gradient
direction for non-binary rewards. By continuity, the same conclusion holds
for all sufficiently small $\epsilon>0$.

\section{Applicable Scenarios}
\label{app:application}

\METHOD{} is particularly well-suited for settings where the training rewards are noisy or unreliable, as the gradients are more dispersed. We highlight three common scenarios, with illustrations in Figure~\ref{fig:scenario}.

\paragraph{Unreliable Reward Signals} In many RL settings, the reward signal can be noisy or unreliable, providing an imperfect proxy for true task performance. For example, in LLM RL, rewards are often obtained from model-based judges, and such rewards can be inaccurate on ambiguous problems or inputs beyond the judge’s capabilities. These cases often exhibit intermediate accuracy and are therefore difficult to remove with accuracy-based filtering, which can waste compute and destabilize training. \METHOD{} is well-suited to this setting: gradients induced by misjudged problems tend to be weakly aligned with the validation gradient (often near-orthogonal in high-dimensional space), so \METHOD{} naturally downweights them. In Section~\ref{sec:experiments}, we show that \METHOD{} improves the stability of the training and the final performance under unreliable reward signals.

\paragraph{Distribution Imbalance} When the training dataset contains a wide range of problems (e.g., math problems containing different categories~\citep{albalak2025big}) and the target downstream task is more specific (e.g., one specific puzzle game), many training problems may be correct and of moderate difficulty, but uninformative for the target task, causing accuracy-based methods to fail at filtering such problems. \METHOD{} addresses this mismatch by favoring problems whose gradients align with the validation gradient, while downweighting target-irrelevant ones with low alignment.

\paragraph{Low-Utility Training Data} As RL training scales up, practitioners increasingly rely on large web-mined datasets~\citep{yue2024mammoth2,du2025generalizable}. Unlike distribution imbalance, where useful data is underrepresented, the challenge here is that many training problems are on-distribution but overly simple, yielding valid rewards while providing limited learning signal for the target task. \METHOD{} addresses this setting by prioritizing problems whose gradients align with validation performance and downweighting low-utility samples.

\section{Measuring Gradient Noise}
\label{app:measure}
We conduct an ablation study to determine the number of response samples required for a reliable estimation of the cosine similarity between policy gradients. Specifically, we estimate the gradient cosine similarity using $k_v \in \{32, 128, 512\}$ samples per problem, computing each estimate twice with different random seeds. We then calculate the Pearson correlation between similarity scores obtained from these independent runs using the DAPO training set and the AMC22 validation set. 

Table~\ref{table:correlation} shows that the correlation increases with sample size, indicating that larger $k_v$ produces more stable and accurate estimates. $k_v=512$ achieves the most accurate estimation with consistent similarity scores but requires high computational cost. $k_v=32$ reduces the computational cost significantly but produces noisy estimates. Since \METHOD{} relies on relative ranking rather than exact gradient values, a moderate correlation (e.g., 0.49) already provides a reliable directional signal in high-dimensional gradient space.



\begin{table}[th]
  \centering
  \small
  \setlength{\tabcolsep}{5pt}
  \begin{tabular}{lccc}
    \toprule
    Number of Samples $k_v$           &  32  & 128  & 512   \\
    \midrule
    Cosine Similarity          & 0.30 & 0.49 & 0.79 \\
    \bottomrule
  \end{tabular}
  \vspace{0.5mm}
  \caption{\textbf{Pearson Correlation of Similarity Scores Calculated Using Responses Sampled Twice.} Pearson correlation between cosine similarity scores computed using two different random seeds, for varying numbers of samples per problem ($k_v$). Samples and gradients are produced from Qwen2.5-1.5B-Math-Instruct.
  }\label{table:correlation}
  \vspace{-4mm}
\end{table}

\section{Compute and Memory Analysis}
\label{app:cost}

We provide a detailed per-round FLOP breakdown and a memory analysis for the computational overhead summarized in Section~\ref{sec:overhead}.

\paragraph{FLOP accounting.}
We estimate per-round FLOPs following~\citet{kaplan2020scaling}, where a forward pass costs $\approx 2N$ FLOPs per token and a forward-backward pass $\approx 6N$, with $N$ the model size and $T$ the sequence length. In our standard setting ($|\mathcal{P}_r|=5120$, $|\mathcal{S}_r|=1280$, $k_r=k_v=16$, $n_t=128$, $|\mathcal{P}_v|=200$), the per-round cost decomposes as:
\begin{itemize}
    \item \textbf{Candidate rollouts:} $|\mathcal{P}_r|\,k_r\cdot 2N T \approx 2.6\times 10^{18}$ FLOPs;
    \item \textbf{Candidate gradients:} $|\mathcal{P}_r|\,k_r\cdot 6N T \approx 7.9\times 10^{18}$ FLOPs;
    \item \textbf{Validation gradient:} $|\mathcal{P}_v|\,k_v\cdot 6N T \approx 3.1\times 10^{17}$ FLOPs (negligible);
    \item \textbf{RL update:} $|\mathcal{S}_r|\,n_t\cdot 6N T \approx 1.6\times 10^{19}$ FLOPs.
\end{itemize}
The three \METHOD{} terms (candidate rollouts, candidate gradients, and validation gradient) total $\approx 1.1\times 10^{19}$ FLOPs against a $\approx 1.6\times 10^{19}$-FLOP update, matching the $\approx 65\%$ relative overhead reported in Section~\ref{sec:overhead}. Since $N$ appears linearly in every term, it cancels in the ratio, so this relative overhead is \textbf{invariant to model scale}.

\paragraph{Memory.}
\METHOD{} introduces no new memory peak beyond a standard GRPO step. A GRPO update under AdamW must hold model weights, activations, and optimizer states, which together occupy roughly $16N$ at peak~\citep{rajbhandari2020zero}. During the selection step, \METHOD{} only stores the gradient of the current candidate and the aggregated validation gradient; combined with the model parameters, this is about $6N$, well below the training peak. Gradient alignment thus reuses the existing training footprint rather than adding to it.


\end{document}